\documentclass{article}
\usepackage[accepted]{icml2018}
\usepackage[utf8]{inputenc}
\usepackage{amsmath, amsfonts, amsthm}
\usepackage{graphicx}
\usepackage{xcolor}
\usepackage[colorlinks=True, citecolor=blue]{hyperref}

\theoremstyle{definition}
\newtheorem{myprp}{Proposition}
\newtheorem{mythm}{Theorem}

\graphicspath{{figs/}}
\newcommand{\V}{\ve{V}}
\newcommand{\W}{\ve{W}}
\renewcommand{\H}{\ve{H}}
\def\bal#1{\begin{align}#1\end{align}}

\newcommand{\ve}[1]{ {\mathbf{#1}} }
\def\defequal{\stackrel{\mbox{\footnotesize def}}{=}}
\newcommand{\bbeta}{\boldsymbol{\beta}}
\newcommand{\balpha}{\boldsymbol{\alpha}}

\begin{document}

\twocolumn[
\icmltitle{Closed-form Marginal Likelihood in Gamma-Poisson Matrix Factorization}

\begin{icmlauthorlist}
\icmlauthor{Louis Filstroff}{to}
\icmlauthor{Alberto Lumbreras}{to}
\icmlauthor{Cédric Févotte}{to}
\end{icmlauthorlist}

\icmlaffiliation{to}{IRIT, Université de Toulouse, CNRS, France}

\icmlcorrespondingauthor{Louis Filstroff}{louis.filstroff@irit.fr}

\icmlkeywords{Matrix factorization, Dictionary Learning}

\vskip 0.3in
]

\printAffiliationsAndNotice

\begin{abstract}
We present novel understandings of the Gamma-Poisson (GaP) model, a probabilistic matrix factorization model for count data. We show that GaP can be rewritten free of the score/activation matrix. This gives us new insights about the estimation of the topic/dictionary matrix by maximum marginal likelihood estimation. In particular, this explains the robustness of this estimator to over-specified values of the factorization rank, especially its ability to automatically prune irrelevant dictionary columns, as empirically observed in previous work. The marginalization of the activation matrix leads in turn to a new Monte Carlo Expectation-Maximization algorithm with favorable properties.
\end{abstract}

\section{Introduction}

The Gamma-Poisson (GaP) model is a probabilistic matrix factorization model which was introduced in the field of text information retrieval \citep{canny2004gap, buntine2006discrete}. In this field, a corpus of text documents is typically represented by an integer-valued matrix \textbf{V} of size $F \times N$, where each column $\textbf{v}_n$ represents a document as a so-called ``bag of words''. Given a vocabulary of $F$ words (or in practice semantic stems), the matrix entry $v_{fn}$ is the number of occurrences of word $f$ in the document $n$. GaP is a generative model described by a dictionary of ``topics'' or ``patterns'' $\W$ (a non-negative matrix of size $F\times K$) and a non-negative ``activation'' or ``score'' matrix $\H$ (of size $K \times N$), as follows:
\begin{align}
\H & \sim \prod_{k,n} \text{Gamma}(h_{kn} | \alpha_{k},\beta_{k}) \label{eqn:prior}, \\
\V | \H &\sim \prod_{f,n} \text{Poisson}\left(v_{fn} | [\W \H]_{fn}\right) \label{eqn:obs},
\end{align}
where we use the shape and rate parametrization of the Gamma distribution, i.e.,
$\text{Gamma}(x|\alpha,\beta) = \frac{\beta^{\alpha}}{\Gamma(\alpha)} x^{\alpha-1} e^{-\beta x}$. The dictionary \textbf{W} is treated as a free deterministic variable.

Though this generative model takes its origins in text information retrieval, it has found applications (with variants) in other areas such as image reconstruction \citep{cemgil2009bayesian}, collaborative filtering \citep{ma2011probabilistic, gopalan2015scalable} or audio signal processing \citep{virtanen2008bayesian}.

Denoting $\boldsymbol{\alpha} = [\alpha_1, \ldots, \alpha_K]^{T}$, $\boldsymbol{\beta} = [\beta_1, \ldots, \beta_K]^{T}$, and treating the shape parameters $\alpha_{k}$ as fixed hyperparameters, maximum joint likelihood estimation (MJLE) in GaP amounts to the minimization of
\bal{
C_{\text{JL}}(\textbf{W}, \textbf{H}, \bbeta) & \defequal - \log p(\V,\H | \W, \bbeta) \label{eqn:JL1} \\
& = D_{\text{KL}}(\V|\W \H) + R_{\balpha}(\H,\bbeta) + \text{cst} \label{eqn:JL2}
}
where $D_{\text{KL}}(\cdot | \cdot )$ is the generalized Kullback-Leibler (KL) divergence defined by 
\bal{
D_{\text{KL}}(\textbf{V} | \hat{\textbf{V}}) = \sum_{f,n} \left( v_{fn} \log \left(\frac{v_{fn}}{\hat{v}_{fn}} \right) - v_{fn} + \hat{v}_{fn} \right)
}
and 
\bal{
&R_{\balpha}(\H,\bbeta) = \nonumber \\
&  \sum_{k,n} \left[ (1 - \alpha_k) \log(h_{kn}) + \beta_k h_{kn} \right] - N \sum_{k } \alpha_{k} \log \beta_{k}.
}

Equation~\eqref{eqn:JL2} shows that MJLE is tantamount to penalized KL non-negative matrix factorization (NMF) \citep{lee2000algorithms} and may be addressed using alternating majorization-minimization \citep{canny2004gap, fevotte2011algorithms, dikmen2012maximum}.

As explained in \citet{dikmen2012maximum}, MJLE can be criticized from a statistical point of view. Indeed, the number of estimated parameters grows with the number of samples $N$ (this is because $\H$ has as many columns as $\V$). To overcome this issue, they have instead proposed to consider maximum marginal likelihood estimation (MMLE), in which $\H$ is treated as a latent variable over which the joint likelihood is integrated. In other words, MMLE relies on the minimization of
\begin{align}
C_{\text{ML}}(\textbf{W},\bbeta) & \defequal - \log p(\textbf{V}|\textbf{W},\bbeta)  \\
& = - \log \int_{\textbf{H}} p(\textbf{V}|\textbf{H}, \W, \bbeta) p(\textbf{H}|\bbeta) \mathrm{d}\textbf{H}. \label{eqMMLE}
\end{align}

We emphasize that MMLE treats the dictionary $\W$ as a free deterministic variable. This is in contrast with fully Bayesian approaches where $\W$ is given a prior, and where estimation revolves around the posterior $p(\W,\H|\V)$. For instance, \citet{buntine2006discrete} place a Dirichlet prior on the columns of $\mathbf{W}$, while \citet{cemgil2009bayesian} considers independent Gamma priors. \citet{zhou2012beta}, \citet{zhou2015negative} set a Dirichlet prior on the columns of $\mathbf{W}$ and a Gamma-based non-parametric Bayesian prior on $\mathbf{H}$, which allows for possible rank estimation.

\citet{dikmen2012maximum} assumed that a closed-form expression of $C_{\text{ML}}$ was not available. Besides, they proposed variational and Monte Carlo Expectation-Maximization (MCEM) algorithms based on a complete set formed by $\ve{H}$ and a set of other latent components $\ve{C}$ that will later be defined. In their experiments, they found MMLE to be robust to over-specified values of $K$, while MJLE clearly overfit. This intriguing (and advantageous) behavior was left unexplained. In this paper, we provide the following contributions:
\begin{itemize}
\item We provide a computable closed-form expression of $C_{\text{ML}}$. The expression is tedious to compute for large $F$ and $K$, as it involves combinatorial operations, but is workable for reasonably dimensioned problems.
\item We show that the proposed closed-form expression reveals a penalization term on the columns of $\W$ that explains the ``self-regularization'' effect observed in \citet{dikmen2012maximum}.
\item {We show that the marginalization of $\H$ allows to derive a new MCEM algorithm with favorable properties.}
\end{itemize}

The rest of the paper is organized as follows. Section 2 introduces preliminary material (composite form of GaP, useful probability distributions). In Section 3, we propose two new parameterizations of the GaP model in which $\H$ has been marginalized. This yields a closed-form expression of $C_{\text{ML}}$ which is discussed in Section 4. Finally, a new MCEM algorithm is introduced in Section 5 and is compared to the MCEM algorithms proposed in \citet{dikmen2012maximum} on synthetic and real data.

\section{Preliminaries}

\subsection{Composite structure of GaP}

GaP can be written as a composite model, thanks to the superposition property of the Poisson distribution \citep{fevotte2009nonnegative}:
\begin{align}
h_{kn} & \sim \text{Gamma}(\alpha_k, \beta_k) \label{eqn:gapc1} \\ 
\ve{c}_{kn} | h_{kn} & \sim \prod_{f} \text{Poisson}(c_{fkn} | w_{fk} h_{kn}) \label{eqn:gapc2} \\ 
\ve{v}_{n} & = \sum_k \ve{c}_{kn} \label{eqn:gapc3}
\end{align}

The vectors $\ve{c}_{kn} = [c_{1kn}, \ldots, c_{Fkn}]^{T}$ of size $F$ and which sum up to $\ve{v}_{n}$ are referred to as {\em components}. In the remainder, \textbf{C} will denote the $F \times K \times N$ tensor with coefficients $c_{fkn}$.  

\subsection{Negative Binomial and Negative Multinomial distributions}

In this section, we introduce two probability distributions that will be used later in the article.

\subsubsection{Negative Binomial distribution}

A discrete random variable $X$ is said to have a negative binomial (NB) distribution with parameters $\alpha > 0$ (called the dispersion or shape parameter) and $p \in [0,1]$ if, for all $c \in \mathbb{N}$, the probability mass function (p.m.f.) of $X$ is given by:
\begin{equation}
\mathbb{P}(X = c) = \frac{\Gamma(\alpha+c)}{\Gamma(\alpha) \, c!} (1-p)^{\alpha} p^c.
\end{equation}

Its variance is $\frac{\alpha p}{(1-p)^2}$, which is larger than its mean, $\frac{\alpha p}{1-p}$. It is therefore a suitable distribution to model over-dispersed count data. Indeed, it offers more flexibility than the Poisson distribution where the variance and the mean are equal.

The NB distribution can be obtained via a Gamma-Poisson mixture, that is:
\begin{align}
\mathbb{P}(X = c) & = \int_{\mathbb{R}_{+}} \text{Poisson}(c|\lambda) \text{Gamma}(\lambda|\alpha,\beta) \mathrm{d} \lambda \\
& = \text{NB} \left(\alpha, \frac{1}{\beta+1}\right).
\end{align}

\subsubsection{Negative Multinomial distribution}

The negative multinomial (NM) distribution \citep{sibuya1964negative} is the multivariate generalization of the NB distribution. It is parametrized by a dispersion parameter $\alpha > 0$ and a vector of event probabilities $\textbf{p} = [ p_1, \ldots, p_{F} ]^{T}$, where $0 \leq p_f \leq 1$ and $\sum_f p_f \leq 1$. Denoting $p_0 = 1 - \sum_f p_f$, and for all $(c_{1}, \ldots, c_{F}) \in \mathbb{N}^{F}$, the p.m.f. of the NM distribution is given by
\begin{equation}
\mathbb{P}(X_1 = c_1, \ldots, X_F = c_F) = \frac{\Gamma(\alpha + \sum_f c_f)}{\Gamma(\alpha) \prod_f c_f!} p_0^{\alpha} \prod_f p_f^{c_f},
\end{equation}
with expectation given by
\begin{equation}
\mathbb{E}(X) = \alpha \left[ \frac{p_1}{p_0}, \ldots, \frac{p_F}{p_0} \right]^{T}.
\end{equation}

In particular, the NM distribution arises in the following Gamma-Poisson mixture, as detailed in the next proposition.

\begin{myprp}
Let ${X} = [X_1, \ldots, X_F]^{T}$ be a random vector whose entries are independent Poisson random variables. Assume that each variable $X_f$ is governed by the parameter $w_f\lambda$, where $\lambda$ is itself a Gamma random variable with parameters $(\alpha, \beta)$. Then the joint probability distribution of ${X}$ is a NM distribution with dispersion parameter $\alpha$ and event probabilities
\begin{equation}
\textbf{p} = \left[ \frac{w_1}{\sum_f w_f + \beta}, \ldots, \frac{w_F}{\sum_f w_f + \beta} \right]^{T}.
\end{equation}
\label{prop1}
\end{myprp}

\begin{proof} $\mathbb{P}(X = [c_1, \ldots, c_F]^{T})$
\begin{align*}
& = \int_{\mathbb{R}_{+}} \mathbb{P}(X_1 = c_1, \ldots, X_F = c_F|\lambda)p(\lambda) d\lambda \\
& = \int_{\mathbb{R}_{+}} \left( \prod_f \frac{(w_f\lambda)^{c_f} e^{-w_f \lambda}}{c_f!} \right) \frac{\beta^{\alpha}}{\Gamma({\alpha})} \lambda^{\alpha-1}e^{-\beta \lambda} d\lambda \\
& = \left( \prod_f \frac{w_f^{c_f}}{c_f!} \right) \frac{\beta^{\alpha}}{\Gamma({\alpha})} \frac{\Gamma(\alpha + \sum_f c_f)}{\left( \sum_f w_f + \beta\right)^{\alpha + \sum_f c_f}} \\
& = \frac{\Gamma(\alpha + \sum_f c_f)}{\Gamma({\alpha}) \prod_f c_f!} \left( \frac{\beta}{\sum_f w_f + \beta} \right)^{\alpha} \prod_f \left( \frac{w_f}{\sum_f w_f + \beta} \right)^{c_f}
\end{align*}
\end{proof}

The NM distribution can also be obtained with an alternative generative process, as shown in the following proposition.

\begin{myprp}
Let ${Y} = [Y_1, \ldots, Y_F]^{T}$ be a random vector following a multinomial distribution with number of trials $L$ and event probabilities $\textbf{p} = [\frac{w_1}{\sum_f w_f}, \ldots, \frac{w_F}{\sum_f w_f}]^{T}$. Assume that $L$ is a NB random variable with dispersion parameter $\alpha$ and probability $p = \frac{\sum_f w_f}{\sum_f w_f + \beta}$.  Then the random vectors ${Y}$ and ${X}$ (as defined in Proposition~\ref{prop1}) have the same distribution.
\label{prop2}
\end{myprp}

\begin{proof}
$\mathbb{P}({Y} = [c_1, \ldots, c_F]^{T})$
\begin{align*}
& = \mathbb{P}(Y = [c_1, \ldots, c_F]^{T} | L) \times \mathbb{P}(L) \\
& = \frac{L!}{\prod_f c_f!} \prod_f \left( \frac{w_f}{\sum_f w_f} \right)^{c_f} \notag \\ 
& \times \frac{\Gamma(\alpha + L)}{\Gamma(\alpha)L!} \left( \frac{\beta}{\sum_f w_f + \beta} \right)^{\alpha}  \left( \frac{\sum_f w_f}{\sum_f w_f + \beta} \right)^{L} \\
\end{align*}
Noting that $L = \sum_f c_f$ completes the proof.
\end{proof}

\section{New formulations of GaP}

We now show how GaP can be rewritten free of the latent variables \textbf{H} in two different ways.

\subsection{GaP as a composite NM model}

\begin{mythm} GaP can be rewritten as follows:
\begin{align}
\textbf{c}_{kn} & \sim \text{NM} \left( \alpha_k, \left[ \frac{w_{1k}}{\sum_f w_{fk} + \beta_k}, \ldots, \frac{w_{Fk}}{\sum_f w_{fk} + \beta_k} \right]^{T} \right) \label{eqn:gapnm} \\
\textbf{v}_{n} & = \sum_k \textbf{c}_{kn}
\end{align}
\label{thm1}
\end{mythm}

\begin{proof}
Combining Equations~\eqref{eqn:gapc1}-\eqref{eqn:gapc3} with Proposition \ref{prop1} completes the proof.
\end{proof}

GaP may thus be interpreted as a composite model in which the $k^{th}$ component has a NM distribution with parameters governed by $\ve{w}_{k}$ (the $k^{th}$ column of $\W$), $\alpha_{k}$ and $\beta_{k}$. Using straightforward computations, the data expectation can be expressed as
\begin{align}
\mathbb{E}(\textbf{v}_n) & = \sum_{k} \mathbb{E}(\textbf{c}_{kn}) \\
& = \sum_{k} \frac{\alpha_{k}}{\beta_{k}} \ve{w}_{k}.
\end{align}

\subsection{GaP as a composite multinomial model}

\begin{mythm}
GaP can be rewritten as follows:
\begin{align}
L_{kn} & \sim \text{NB}\left(\alpha_k, \frac{\sum_f w_{fk}}{\sum_f w_{fk} + \beta_k}\right) \\
\textbf{c}_{kn} & \sim \text{Mult}\left( L_{kn}, \left[ \frac{w_{1k}}{\sum_f w_{fk}}, \ldots, \frac{w_{Fk}}{\sum_f w_{fk}} \right]^{T} \right) \\
\textbf{v}_{n} & = \sum_k \textbf{c}_{kn}\end{align}
where ``$\text{Mult}$'' refers to the multinomial distribution.
\label{thm2}
\end{mythm}

\begin{proof}
Combining Equations~\eqref{eqn:gapc1}-\eqref{eqn:gapc3} with Proposition \ref{prop2} completes the proof.
\end{proof}

Theorem~\ref{thm2} states that another interpretation of GaP consists in modeling the data as a sum of $k$ independent multinomial distributions, governed individually by $\ve{w}_{k}$ and whose number of trials is random, following a NB distribution governed by $\ve{w}_{k}$, $\alpha_{k}$ and $\beta_{k}$.

A special case of the reformulation of GaP offered by Theorem~\ref{thm2} is given by \citet{buntine2006discrete} using a different reasoning, when it is assumed that $\sum_f w_{fk} = 1$ (a common assumption in the field of text information retrieval, where the columns of $\ve{W}$ are interpreted as discrete probability distributions). Theorem~\ref{thm2} provides a more general result as it applies to any non-negative matrix $\ve{W}$.

\section{Closed-form marginal likelihood \label{sec:closed}}

\subsection{Analytical expression}

Until now, it was assumed that the marginal likelihood in the GaP model was not available analytically. However, the new parametrization offered by Theorem \ref{thm1} allows to obtain a computable analytical expression of the marginal likelihood $C_{\text{ML}}$. Denote by $\mathcal{C}$ the set of all ``admissible'' components, i.e.,
\begin{equation}
\mathcal{C} = \{ \textbf{C} \in \mathbb{N}^{F \times K \times N} \, | \, \forall (f, n), \sum\nolimits_k c_{fkn} = v_{fn}\}.
\end{equation}

By marginalization of $\ve{C}$, we may write
\begin{align}
p(\textbf{V}|\textbf{W}, \boldsymbol{\beta}) & = \sum_{\textbf{C} \in \mathcal{C}} p(\textbf{C}|\textbf{W}, \boldsymbol{\beta}) \\
& = \sum_{\textbf{C} \in \mathcal{C}} \prod_{k,n} p(\textbf{c}_{kn}|\textbf{W}, \boldsymbol{\beta}).
\end{align}

Using Equation~\eqref{eqn:gapnm} we obtain
\begin{align}
& p(\textbf{V}|\textbf{W}, \boldsymbol{\beta}) = \notag \\
& \sum_{\textbf{C} \in \mathcal{C}} \prod_{k,n}  \left[ \frac{\Gamma(\sum_f c_{fkn} + \alpha_k)}{\Gamma({\alpha_k}) \prod_f c_{fkn}!} \left( \frac{\beta_k}{\sum_f w_{fk} + \beta_k} \right)^{\alpha_k} \right. \notag \\
& \times \left. \prod_f \left( \frac{w_{fk}}{\sum_f w_{fk} + \beta_{k}} \right)^{c_{fkn}} \right].
\label{eqlkf}
\end{align}

Introducing the notations 
\begin{equation}
\Omega_{\balpha}(\textbf{C}) = \prod_{k,n} \frac{\Gamma(\sum_f c_{fkn} + \alpha_k)}{\Gamma(\alpha_k)\prod_f c_{fkn}!}
\end{equation}
and
\begin{equation}
p_{fk} = \frac{w_{fk}}{\sum_f w_{fk} + \beta_k} \label{eqn:defp}
\end{equation}
we may rewrite Eq.~\eqref{eqlkf} as
\bal{
 p(\textbf{V}|\textbf{W}, \boldsymbol{\beta}) & = \left[ \prod_{k} ( 1-\sum_{f} p_{fk} )^{N\alpha_k} \right] \notag \\
 &\quad  \times  \sum_{\textbf{C} \in \mathcal{C}} \left[ \Omega_{\balpha}(\textbf{C}) \prod_{f,k}  p_{fk} ^{\sum_{n} c_{fkn}} \right]. \label{eqn:ml}
}

Equation~\eqref{eqn:ml} is a computable closed-form expression of the marginal likelihood. It is free of $\H$ and in particular of the integral that appears in Equation~\eqref{eqMMLE}. However the expression~\eqref{eqn:ml} is still semi-explicit because it involves a sum over the set of all admissible components $\ve{C}$. $\mathcal{C}$ is countable set with cardinality $ \# {\cal C} = {\prod_{f,n} \binom{v_{fn}+ K - 1}{K - 1}}$. It is straightforward to construct but challenging to compute in large dimension, and for large values of $v_{fn}$.

The sum over all the matrices in the set $\mathcal{C}$ expresses the convolution of the (discrete) probability distributions of the $K$ components. Unfortunately, the distribution of the sum of independent negative multinomial variables of different event probabilities is not available in closed form.

As already known from \citet{dikmen2012maximum}, the value of the marginal likelihood is unchanged when the scales of the columns of $\W$ and the rates $\bbeta$ are changed accordingly. Let $\boldsymbol{\Lambda}$ be a non-negative diagonal matrix of size $K$, it can easily be derived from Equation~\eqref{eqn:ml} that
\begin{equation}
p(\textbf{V}|\textbf{W} \boldsymbol{\Lambda}, \boldsymbol{\beta}\boldsymbol{\Lambda}) = p(\textbf{V}|\textbf{W}, \boldsymbol{\beta}).
\end{equation}

We therefore have a scaling invariance between $\mathbf{W}$ and $\boldsymbol{\beta}$, and as such, we may fix $\boldsymbol{\beta}$ to  arbitrary values and leave $\textbf{W}$ free. Thus, we will treat $\bbeta$ as a constant in the following and drop it from the arguments of $C_{\text{ML}}$.

\subsection{Self-regularization} 

\citet{dikmen2012maximum} empirically studied the properties of MMLE. In particular, they observed the self-ability of the estimator to regularize the number of columns of $\mathbf{W}$. For example, one experiment consisted in generating synthetic data according to the GaP model, with a ground-truth number of components $K^{\star}$. MMLE was run with $K > K^{\star}$ and they noticed that the estimated $\W$ contained $K - K^{\star}$ empty columns. As such, the estimator was able to recover the ground-truth dimensionality. In contrast, MJLE used all $K$ dimensions and overfit the data. They were unable to give a theoretical justification of the observed phenomenon, but provided a first insight thanks to a Laplace approximation of $p(\textbf{V}|\textbf{W})$. The closed-form expression~\eqref{eqn:ml} offers a deeper understanding of this phenomenon, as explained next.

Using Equations~\eqref{eqn:ml} and \eqref{eqn:defp} and treating $\bbeta$ as a constant, the negative log-likelihood can be expressed as
\begin{align}
& \frac{1}{N} C_{\text{ML}}(\W) = \notag \\
&\quad  - \frac{1}{N} \log \left( \sum_{\textbf{C} \in \mathcal{C}} \Omega_{\balpha}(\textbf{C}) \prod_{f,k} p_{fk}^{\sum_n c_{fkn}} \right)  \label{eqg} \\
&\quad  + \sum_k \alpha_k \log(||\textbf{w}_k||_1 + \beta_k) + \text{cst}, \label{eqg2}
\end{align}
where cst = $-\sum_k \alpha_k \log \beta_k$.

The negative log-likelihood reveals two terms. The first term, Equation~\eqref{eqg}, captures the interaction between data $\V$ (through $\ve{C}$) and the parameter $\W$ (through the event probabilities $p_{fk} = w_{fk}/(\|\ve{w}_{k}\|_{1} + \beta_{k})$). The second term, Equation~\eqref{eqg2}, only depends on the parameter $\W$ and can be interpreted as a group-regularization term. The non-convex and sharply peaked function $f(\ve{x}) = \sum_k \log(x_k + b) $ is known to be sparsity-inducing \citep{candes2008enhancing}. As such, the term~\eqref{eqg2} will promote sparsity of the norms of the columns of $\W$. When a norm $||\textbf{w}_k||_1$ is set to zero for some $k$, the whole column $\textbf{w}_k$ is set to zero because of the non-negativity constraint. This gives a formal explanation of  the ability of MMLE to automatically prune columns of \textbf{W}, without any explicit sparsity-inducing prior at the modeling stage (recall that $\ve{W}$ is a deterministic parameter without a prior).

\section{MCEM algorithms form MMLE}

\subsection{Expectation-Maximization}

We now turn to the problem of optimizing Equation~\eqref{eqMMLE} by leveraging on the results of Section~\ref{sec:closed}. Despite obtaining a closed-form expression, the direct optimization of the marginal likelihood remains difficult. However, the structure of GaP makes Expectation-Maximization (EM) a natural option \citep{dempster1977maximum}. Indeed, GaP involves observed variables $\textbf{V}$ and latent variables \textbf{C} and \textbf{H}. As such, we can derive several EM algorithms based on various choices of the complete set. More precisely, we consider three possible choices that each define a different algorithm, as follows.

{\bf EM-CH.} The complete set is $\{\textbf{C},\textbf{H}\}$ and EM consists in the iterative minimization w.r.t $\W$ of the functional defined by
	\begin{equation}
	Q_{\text{CH}}(\textbf{W}|{\tilde{\textbf{W}}}) = -\int_{\textbf{C}, \textbf{H}} \log p(\textbf{C}, \textbf{H}|\textbf{W}) p(\textbf{C},\textbf{H}|\textbf{V},\tilde{\textbf{W}}) \mathrm{d} \textbf{C} \mathrm{d} \textbf{H},
	\label{Q-CH}
	\end{equation}
where $\tilde{\W}$ is the current estimate. Note that $\ve{V}$ does not need to be included in the complete set because we have $\ve{V} = \sum_{k} \ve{C}_{k}$. This corresponds to the general formulation of EM in which the relation between the complete set and the data is a many-to-one mapping and slightly differs from the more usual one where the complete set is formed by the union of data and a hidden set \citep{dempster1977maximum,fevotte2011efficient}. 

{\bf EM-H.} The complete set is $\{\textbf{V},\textbf{H}\}$ and EM consists in the iterative minimization of
	\begin{equation}
	Q_{\text{H}}(\textbf{W}|{\tilde{\textbf{W}}}) = - \int_{\textbf{H}} \log p(\textbf{V},\textbf{H}|\textbf{W}) p(\textbf{H}|\textbf{V},\tilde{\textbf{W}}) \mathrm{d}\textbf{H}.
	\label{Q-H}
	\end{equation}

{\bf EM-C.} The complete set is merely $\{\textbf{C}\}$ and EM consists in the iterative minimization of 	 
	\begin{equation}
	Q_{\text{C}}(\textbf{W}|{\tilde{\textbf{W}}}) = - \int_{\textbf{C}} \log p(\textbf{C}|\textbf{W}) p(\textbf{C}|\textbf{V},\tilde{\textbf{W}}) \mathrm{d} \textbf{C}.
	\label{Q-C}
	\end{equation}

EM-CH and EM-H have been considered in \citet{dikmen2012maximum}. EM-C is a new proposal that exploits the results of Section~\ref{sec:closed}. In all three cases, the posteriors of the latent variables involved -- $p(\textbf{C},\textbf{H}|\textbf{V},\tilde{\textbf{W}})$, $p(\textbf{H}|\textbf{V},\tilde{\textbf{W}})$ and $p(\textbf{C}|\textbf{V},\tilde{\textbf{W}})$ -- are untractable and neither are the integrals involved in Equations~\eqref{Q-CH}, \eqref{Q-H} and \eqref{Q-C}. To overcome this problem, we resort to Monte Carlo EM (MCEM) \citep{wei1990monte} as described in the next section.

\subsection{Monte Carlo E-step \label{sec:mc-estep}}

MCEM consists in using a Monte Carlo (MC) approximation of the integrals in Equations~\eqref{Q-CH}, \eqref{Q-H} and \eqref{Q-C} based on samples drawn from the posterior distributions $p(\textbf{C},\textbf{H}|\textbf{V},\tilde{\textbf{W}})$, $p(\textbf{H}|\textbf{V},\tilde{\textbf{W}})$ and $p(\textbf{C}|\textbf{V},\tilde{\textbf{W}})$. These can be obtained by Gibbs sampling of the joint posterior $p(\textbf{C},\textbf{H}|\textbf{V},\tilde{\textbf{W}})$, which also returns samples from the marginals $p(\textbf{H}|\textbf{V},\tilde{\textbf{W}})$ and $p(\textbf{C}|\textbf{V},\tilde{\textbf{W}})$ at convergence. The Gibbs sampler can easily be derived because the conditional distributions  $p(\textbf{H}| \textbf{C}, \ve{V}, \tilde{\textbf{W}}) =p(\textbf{H}| \textbf{C},\tilde{\textbf{W}})$ and $p(\textbf{C}|\textbf{H}, \textbf{V},\tilde{\textbf{W}})$ are available in closed form.

At iteration $j+1$, Gibbs sampling writes
\begin{align}
h_{kn}^{(j+1)} &\sim \text{Gamma}(\alpha_k + \sum_f c^{(j)}_{fkn},~\beta_k + \sum_f \tilde{w}_{fk}) \\
\underline{\textbf{c}}_{fn}^{(j+1)}& \sim \text{Mult} \left( v_{fn}, \left[\rho^{(j+1)}_{f1} ,~\ldots, \rho^{(j+1)}_{fK} \right]^{T} \right)
\end{align}
where $\underline{\textbf{c}}_{fn}$ denotes the vector $[c_{f1n}, \ldots, c_{fKn}]^{T}$ of size $K$ and
\begin{equation}
\rho^{(j+1)}_{fk} = \frac{\tilde{w}_{fk}h^{(j+1)}_{kn}}{\sum_{k} \tilde{w}_{fk} h^{(j+1)}_{kn}}.
\end{equation}

Note that $\underline{\textbf{c}}_{fn}^{(j+1)}$ only needs to be sampled when $v_{fn} \neq 0$, since $\underline{\textbf{c}}_{fn}^{(j+1)} = [0, \ldots, 0]^{T}$ when $v_{fn} = 0$.

\subsection{M-step}

Given a set of $J$ samples $\{ \textbf{H}^{(j)}, \textbf{C}^{(j)} \}$ drawn from $p(\textbf{C},\textbf{H}|\textbf{V},\tilde{\textbf{W}})$ (after burn-in), minimization of the MC approximation of $Q_{\text{CH}}$ in Equation~\eqref{Q-CH} yields
\begin{equation}
w^{\text{MCEM-CH}}_{fk} = \frac{\sum_{j,n} c^{(j)}_{fkn}}{\sum_{j,n} h^{(j)}_{kn}},
\end{equation}
as shown by \citet{dikmen2012maximum}. They also show that the following multiplicative update decreases the MC approximation of $Q_{\text{H}}$ in Equation~\eqref{Q-H} at every iteration
\begin{equation}
w^{\text{MCEM-H}}_{fk} = \tilde{w}_{fk} \frac{\sum_{j,n} h^{(j)}_{kn} v_{fn} {[\tilde{\textbf{W}}\textbf{H}^{(j)}]_{fn}^{-1}}}{\sum_{j,n} h^{(j)}_{kn}}.
\end{equation}

We now derive the new update for EM-C. The MC approximation of ${Q}_{\text{C}}$ in Equation~\eqref{Q-C} is given by:
\begin{equation}
 \hat{Q}_{\text{C}}(\textbf{W}|{\tilde{\textbf{W}}}) \defequal - \frac{1}{J} \sum_{j=1}^{J} \log p(\textbf{C}^{(j)}|\textbf{W}).
\end{equation}

Replacing $p(\textbf{C}^{(j)}|\textbf{W})$ by its expression given by Equation~\eqref{eqn:gapnm}, we obtain:
\begin{align}
& \hat{Q}_{\text{C}}(\textbf{W}|{\tilde{\textbf{W}}}) = \frac{1}{J} \sum_{j,k,n} \left[ \alpha_k \log \left( \sum_f w_{fk} + \beta_k \right) \right. \notag \\
& \left. + \sum_{f} c_{fkn}^{(j)} \left(   \log \left(\sum\nolimits_f w_{fk} + \beta_k \right) - \log \left( w_{fk} \right) \right) + \text{cst} \right].
\label{eqqb}
\end{align}

The minimization of $\hat{Q}_{\text{C}}$ w.r.t. \textbf{W} leads to $K$ linear systems of equations that we need to solve for each column $\textbf{w}_k$:
\begin{equation}
\textbf{A}_k \textbf{w}_k = \textbf{b}_k.
\end{equation}

The matrix $\textbf{A}_k \in \mathbb{R}^{F \times F}$ is defined by:
\begin{equation}
a_{fg} = \left( JN \alpha_k + \sum_{j,f,n} c^{(j)}_{fkn} \right) \delta_{fg} - \sum_{j,n} c^{(j)}_{fkn},
\end{equation}
where $\delta_{fg}$ is the Kronecker symbol, i.e. $\delta_{fg} = 1$ if and only if $f = g$, and zero otherwise. The vector $\textbf{b}_k \in \mathbb{R}^{F \times 1}$ is defined by:
\begin{equation}
b_{fk} = \beta_k \sum_{j,n} c^{(j)}_{fkn}.
\end{equation}

The matrix $\ve{A}_k$ appears to be the sum of a diagonal matrix with a rank-1 matrix and can be inverted analytically thanks to the Sherman-Morrison formula \citep{sherman1950adjustment}. This results in the new closed-form update
\begin{equation}
w^{\text{MCEM-C}}_{fk} = \frac{1}{JN} \frac{\beta_k}{\alpha_k}  \sum_{j,n} c^{(j)}_{fkn}.
\end{equation}

In the common case where $\frac{\beta_k}{\alpha_k}$ is equal to a constant $\gamma$ for all $k$, we can write:
\begin{equation}
\sum_k w^{\text{MCEM-C}}_{fk} = \sum_k \frac{\beta_k}{\alpha_k} \frac{\sum_{j,n} c^{(j)}_{fkn}}{JN} = \gamma \frac{\sum_{j,n} v_{fn}}{JN}  = \gamma \overline{v}_f 
\label{eq_cst}
\end{equation}
where $\overline{v}_f = N^{-1} \sum_n v_{fn}$ is the empirical mean of the data for the feature $f$. Equation~\eqref{eq_cst} implies that the rows of the estimate $\W^{\text{MCEM-C}}$ at every iteration sum to a constant. This behavior is illustrated on Figure~\ref{fig1}.

\begin{figure}[t]
	\centering
	\includegraphics[scale=0.475]{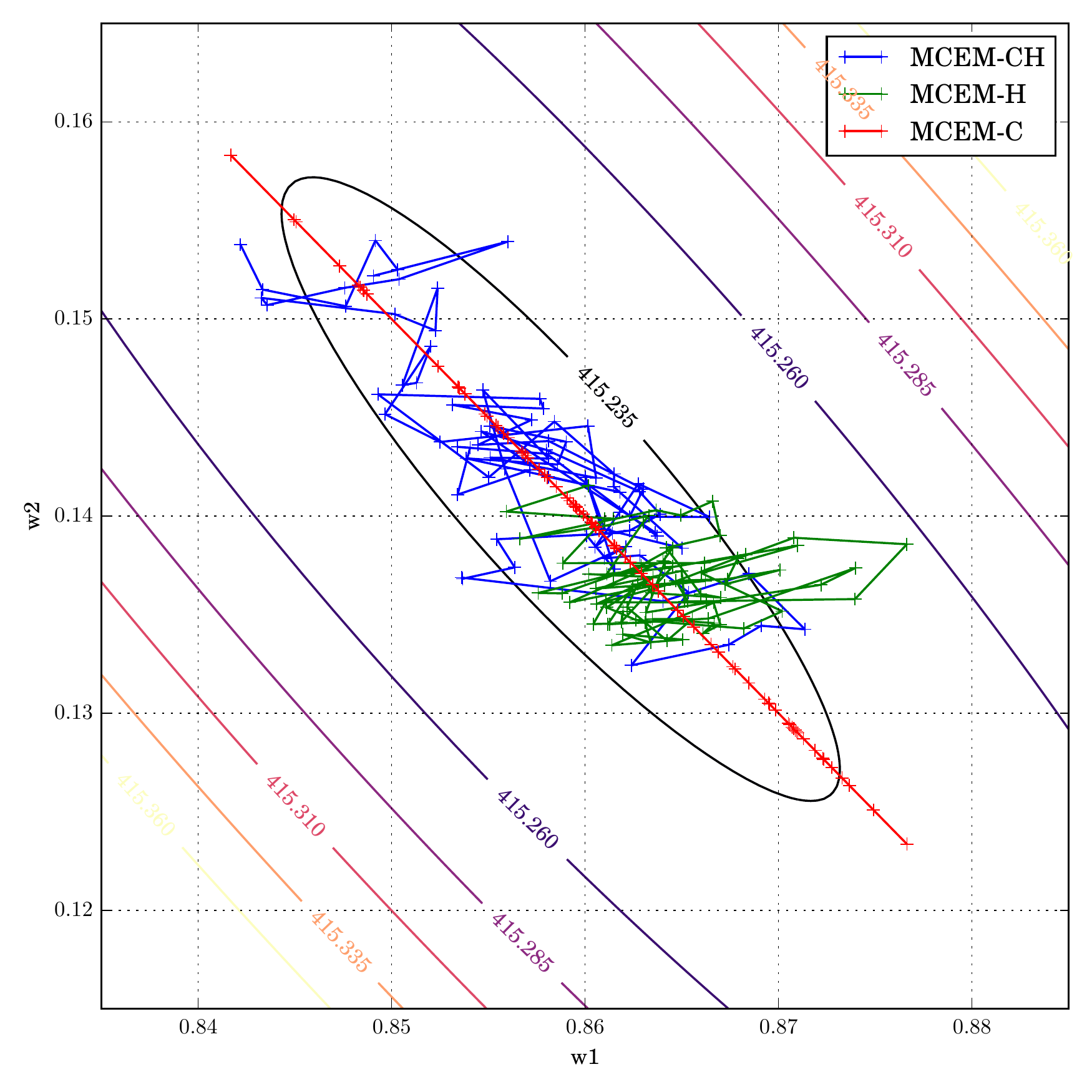}
	\caption{Illustrative example with $F = 1$ and $K =2$. A dataset of $N = 300$ (scalar) samples is generated from GaP. The contour plot displays $C_{\text{ML}}(\ve{W})$ (thanks to its closed-form expression). We set $\alpha_k = 1$ and $\beta_k = 1$. The red, green and blue lines display the iterates of $\ve{W} = [w_1, w_2]$ for the last 80 iterations (out of 500) of MCEM-C, MCEM-H and MCEM-CH, respectively. The iterates generated by MCEM-C satisfy the constraint $w_1 + w_2 = \bar{v}$, as described by Equation~\eqref{eq_cst}.
	\label{fig1}} 
\end{figure}

\section{Experimental results}

We now compare the three MCEM algorithms proposed for MMLE in the GaP model, first using synthetic toy datasets, then real-world data. Python implementations of the three algorithms are available from the first author website.

\subsection{Experiments with synthetic data}

We generate a dataset of $N = 100$ samples according to the GaP model, with the following parameters:
\begin{equation}
\textbf{W}_{1}^{\star} =
\begin{bmatrix} 
0.638 & 0.075 \\
0.009 & 0.568 \\
0.044 & 0.126 \\
0.309 & 0.231 \\
\end{bmatrix}
,~\boldsymbol{\alpha}^{\star} = \boldsymbol{\beta}^{\star} = 1.
\end{equation}
The columns of $\textbf{W}_{1}^{\star}$ have been generated from a Dirichlet distribution (with parameters 1). The generated dataset (of size $4 \times 100$) is denoted by $\textbf{V}_1$.

We proceed to estimate the dictionary $\W$ using hyperparemeters $K = K^\star+1= 3$, $\alpha_k = \beta_k = 1$ with MCEM-C, MCEM-H and MCEM-CH. The algorithms are run for 500 iterations. 300 Gibbs samples are generated at each iteration, with the first 150 samples being discarded for burn-in (this proves to be enough in practice), leading to $J=150$. The Gibbs sampler at EM iteration $i+1$ is initialized with the last sample obtained at EM iteration $i$ (warm restart). The algorithms are initialized from the same deterministic starting point given by 
\begin{equation}
w_{fk} = \frac{1}{K}\frac{\beta_k}{\alpha_k} \overline{v}_f,
\label{eq:i}
\end{equation}
as suggested by Equation~\eqref{eq_cst}.

Figure~\ref{fig2} displays the negative log-likelihood $C_{\text{ML}}(\mathbf{W})$, computed thanks to the derivations of Section~\ref{sec:closed}, and Figure~\ref{fig3}-(a) displays the norm of the three columns of the iterates, both w.r.t. CPU time in seconds. The three algorithms have almost identical computation times, most of the computational burden residing in the Gibbs sampling procedure that is common to the three algorithms. Moreover, the three algorithms converge to the same point, with MCEM-C converging marginally faster than the other two in this case.

\begin{figure}[t]
	\centering
	\includegraphics[scale=0.43]{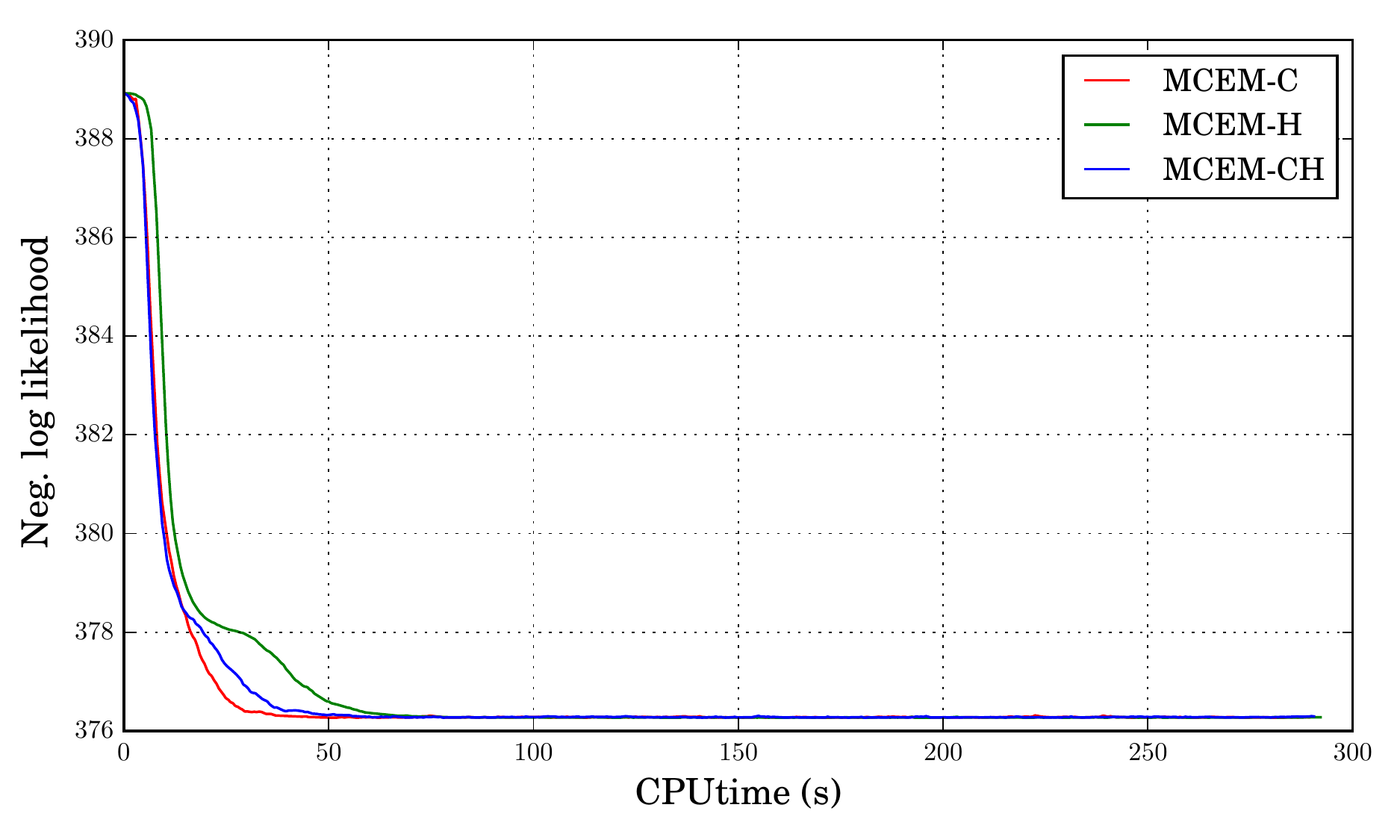}
	\caption{$C_{\text{ML}}(\mathbf{W})$ with respect to CPU time in seconds for the three MCEM algorithms (dataset $\textbf{V}_1$).
	\label{fig2}}
\end{figure}

 Then we proceed to generate a second dataset $\mathbf{V}_2$ according to the GaP model, with now $\textbf{W}_{2}^{\star} = 100 \times \textbf{W}_{1}^{\star}$. The expectation of $\mathbf{V}_2$ is thus a hundred times the expectation of $\mathbf{V}_1$. We apply the exact same experimental protocol to $\mathbf{V}_2$ as we did for $\mathbf{V}_1$, except that the algorithms are now run for a larger number of $2000$ iterations. In this case, because of the combinatorial nature of $\# {\cal C}$, it is impossible to compute the likelihood in reasonable time. The norms of the columns of the iterates are displayed on Figure~\ref{fig3}-(b). As we can see, MCEM-C clearly outperforms the other two algorithms in this scenario. This behavior has been consistently found when estimating dictionaries from datasets with sufficiently large values. This drastic difference in convergence is unexplained at this stage. However, we conjecture that this is linked to the over-dispersion of the data (i.e. when the variance is greater than the mean), which increases with the scale of \textbf{W}. This will be studied in future work.

\begin{figure}[t]
	\begin{center}
		\begin{tabular}{c}
			(a) $\W^{\star}_{1}$ \\
			\includegraphics[scale=0.43]{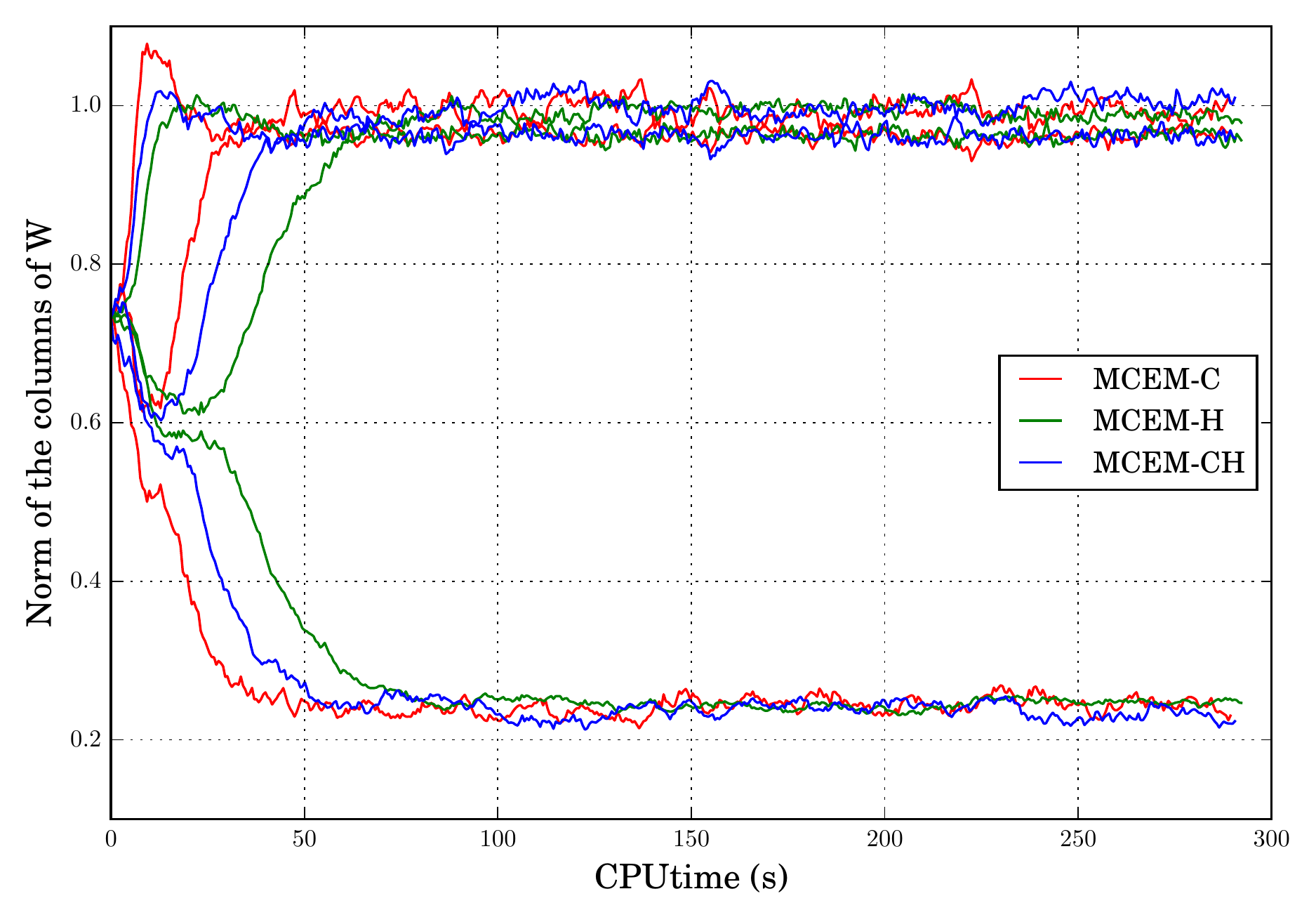} \\
			(b) $\W^{\star}_{2}$ \\	
			\includegraphics[scale=0.43]{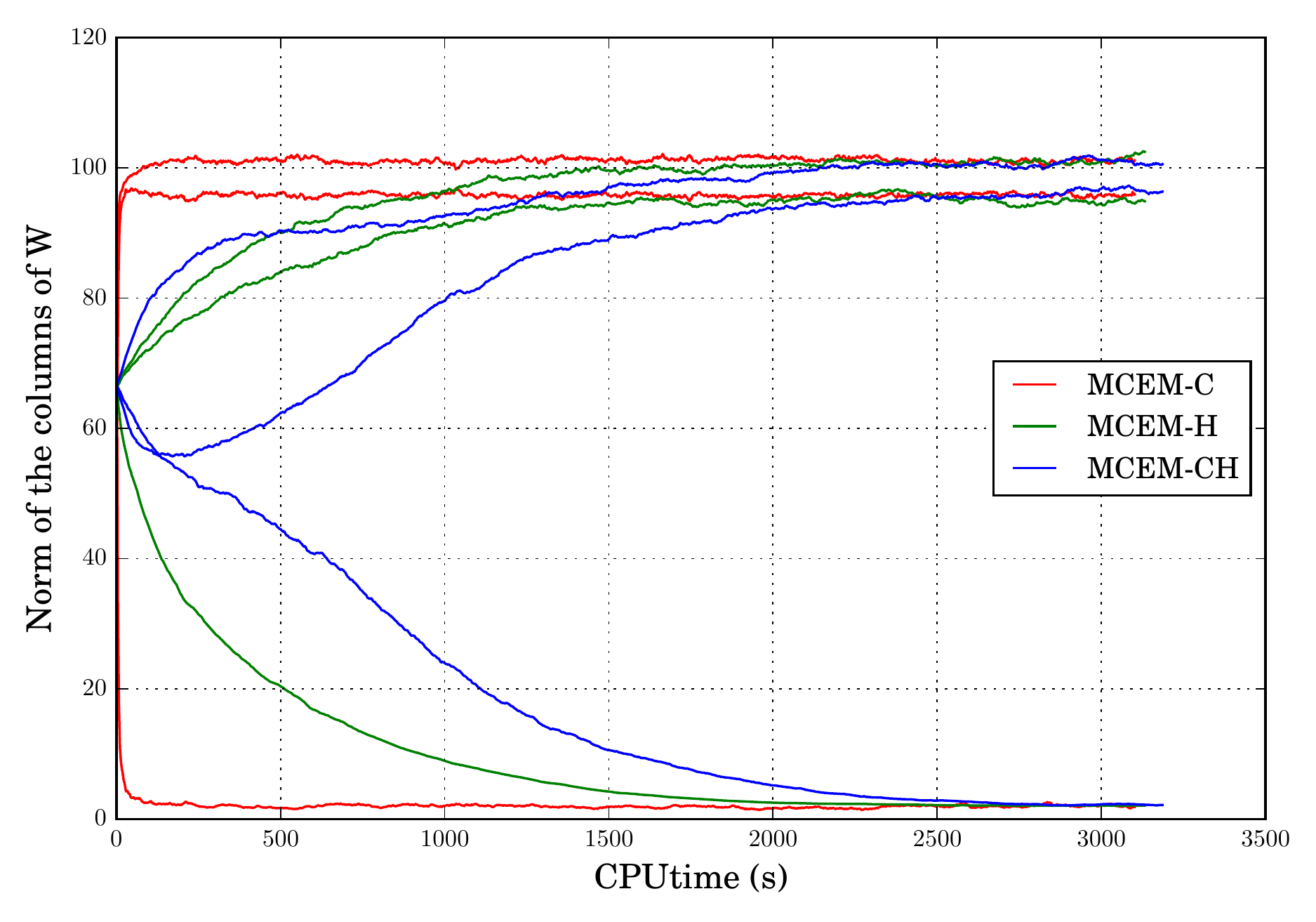}
		\end{tabular}
	\end{center}
	\caption{Evolution of the norm of each of the $K =3$ columns $\textbf{w}_k$ of the dictionaries w.r.t. CPU time for the three MCEM algorithms. Top : dataset $\mathbf{V}_1$, generated according to the GaP model with $\mathbf{W}_1^{\star}$. Bottom : dataset $\mathbf{V}_2$, generated according to the GaP model with $\mathbf{W}_2^{\star}$.
	\label{fig3}}
\end{figure}

\subsection{Experiments with the NIPS dataset}

Finally, we consider the {NIPS} dataset which contains word counts from a collection of articles published at the NIPS conference.\footnote{\url{https://archive.ics.uci.edu/ml/datasets/bag+of+words}} The number of articles is $N=1,500$ and the number of unique terms (appearing at least 10 times after tokenization and removing stop-words) is $F=12,419$. The matrix $\V$ is quite sparse as $96\%$ of its coefficients are zeros. This saves a large amount of computational effort, because we only need to sample $\underline{\mathbf{c}}_{fn}$ for pairs $(f,n)$ such that $v_{fn}$ is non-zero. Moreover, the count values range from 0 to 132.

We applied MCEM-C and MCEM-CH with $K=10$ and $\alpha_k=\beta_k=1$. The algorithms are run for $1,000$ iterations. 250 Gibbs samples are generated in each iteration with the first half being discarded for burn-in (i.e. $J = 125$). The Gibbs sampler at iteration $i+1$ is again initialized with warm restart. The algorithms are initialized using Equation~\eqref{eq:i}. MCEM-H results in similar performance than MCEM-CH and is not reported here.

Figure \ref{fig4} shows that the column norms of the iterates w.r.t. CPU time in seconds. The difference in convergence speed between the two algorithms is again striking. MCEM-C efficiently explores the parameter space in the first iterations and converges dramatically faster than MCEM-CH. The algorithms here converge to different solutions, which confirms the non-convexity of $C_{\text{ML}}(\W)$. Other runs confirmed that MCEM-C is consistently faster, and also that the two algorithms do not always converge to the same solution.

MCEM-C still takes a few hours to converge. We implemented stochastic variants of the EM algorithm such as SAEM \citep{kuhn2004coupling}, however this did not result in significant improvements in our case. Finally, note that large-scale implementations are possible using for example on-line EM \citep{cappe2009line}.

\begin{figure}[t]
	\centering
	\includegraphics[scale=0.43]{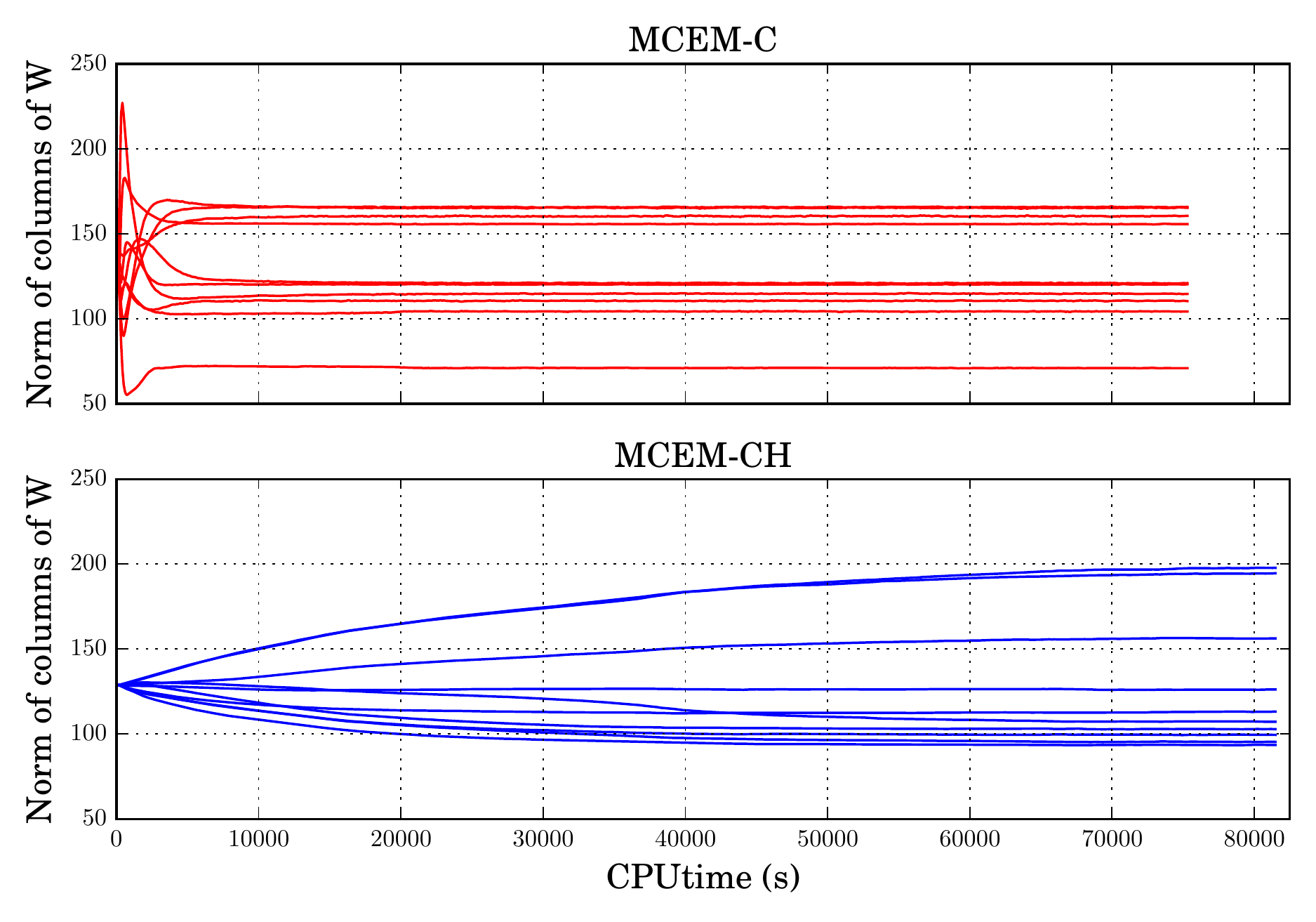}
	\caption{Evolution of the norm of each of the $K =10$ columns of the dictionaries w.r.t. CPU time for MCEM-C and MCEM-CH (NIPS dataset). 
	\label{fig4}}
\end{figure}

\section{Conclusion}

In this paper, we have shown how the Gamma-Poisson model can be rewritten free of the latent variables \textbf{H}. This new parametrization enabled us to come up with a closed-form expression of the marginal likelihood, which revealed a penalization term explaining the ``self-regularization'' effect described in \citet{dikmen2012maximum}. We then proceeded to compare three MCEM algorithms for the task of maximizing the likelihood, and the algorithm taking advantage of the marginalization \textbf{H} has been empirically proven to have better convergence properties.

In this work, we treated $\boldsymbol{\alpha}$ as a fixed hyperparameter. Future work will focus on lifting this hypothesis, and on designing algorithms to estimate both $\textbf{W}$ and $\boldsymbol{\alpha}$. We will also look into carrying a similar analysis in other probabilistic matrix factorization models, such as the Gamma-Exponential model \citep{dikmen2011nonnegative}.

\section*{Acknowledgments}

We thank Benjamin Guedj, V\'ictor Elvira and Jérôme Idier for fruitful discussions related to this work. This work has received funding from the European Research Council (ERC) under the European Union’s Horizon 2020 research and innovation program under grant agreement No 681839 (project FACTORY).

\bibliographystyle{icml2018}
\bibliography{TRBib.bib}

\end{document}